\documentclass[smallextended]{svjour3} 

\smartqed  

\usepackage[ruled,vlined]{algorithm2e}

\usepackage{mathtools}
\usepackage{diagbox}
\usepackage{color, colortbl}
\usepackage{multirow}
\usepackage{algpseudocode}
\usepackage{amsmath,amsfonts,bm}
\usepackage{amssymb}
\usepackage{graphics}
\usepackage{epsfig}

\DeclareMathOperator*{\argmax}{argmax}
\usepackage{mathtools}

\usepackage{xcolor}
\definecolor{tiffany}{rgb}{.039, .729, .710}

\begin{document}

\title{Fast Multi-Class Probabilistic Classifier by Sparse Non-parametric Density Estimation}

\author{
  Wan-Ping Nicole Chen \and Yuan-chin Ivan Chang
}

\institute{Wan-Ping Nicole Chen \at
              Institute of Statistical Science, Academia Sinica,Taipei, Taiwan \\
              \email{wpchen@stat.sinica.edu.tw}             \\
Yuan-chin Ivan Chang \at
		Institute of Statistical Science, Academia Sinica, Taipei, Taiwan\\
		\email{ycchang@sinica.edu.tw}
}

\maketitle

\begin{abstract}
The model interpretation is essential in many application scenarios and to build a classification model with a ease of model interpretation may provide useful information for further studies and improvement.  
It is common to encounter with a lengthy set of variables in modern data analysis, especially when data are collected in some automatic ways.
This kinds of datasets may not collected with a specific analysis target and usually contains redundant features, which have no contribution to a the current analysis task of interest. 
Variable selection is a common way to increase the ability of model interpretation and is popularly used with some parametric classification models.  There is a lack of studies about variable selection in nonparametric classification models such as the density estimation-based methods and this is especially the case for multiple-class classification situations.
In this study we study multiple-class classification problems using the thought of sparse non-parametric density estimation and propose a method for identifying high impacts variables for each class.  We present the asymptotic properties and the computation procedure for the proposed method together with some suggested sample size.  
We also repost the numerical results using both synthesized and some real data sets.

\end{abstract}

\section{Introduction}
\label{intro}
Classification is a common task in all kinds of data analysis scenarios such as medical studies, commercial activities, industrious manufacture research and so on.  In many applications, the accuracy is just a basic requirement to asses a classification rule and the model-interpretation is essential for the follow-up procedures.  For example, the treatment strategy after a medical diagnostic decision may depend on the features that actually affect such a decision.  Besides, a classification rule is usually build on a combination of many features, the subjects being assigned to each sub-group may still have different feature values.  To enhance the customized or adaptive services, which are hot and important research topics, will rely on such information. To this end, to know why or how those subjects are assigned to each subgroup or what variables play important roles in such a decision-making is crucial. This task is more difficult in a multiple-class classification situation and more important in modern classification applications, in which binary classifiers cannot serve well.
We study multiple-class classification problems, where each training point belongs to only one of $c$ ($\geq 2$) different classes. Our goal is to construct a classification function using these training data such that we can correctly assign newly observed subjects to one of the class, while to learn what the important features specifying each class are.

When the densities of each class were known and assuming a person is in a  class, we are able to compute the conditional probability for this person to belong to such a class. Then through comparing these conditional probabilities of each class, we can assign the subject to the class with maximum probability. This classification rule is known as the Bayes rule, and it is known to have the minimum error rate.  
In practice, the information of these densities are usually not available.  Hence, it is natural think of how to construct classification rule based on (non)-parametric density estimation methods.  There are many discussions of this longstanding problem in the literature such as 
\cite{ryzin1966}, 
\cite{DBLP:journals/tit/DevroyeW76}, 
\cite{DBLP:journals/prl/GreblickiP83}, 
\cite{Har-Peled:2002} and
\cite{John95estimatingcontinuous}. 
In particular, the result of 
\cite{ryzin1966} suggests that if we can accurately estimate the density of each class, then the classification error rate approaches to that of the optimal Bayes rule.  Following his results, there are many multi-class approaches based on different density estimation methods; for example, see  \cite{Ancukiewicz1998}, \cite{hall2005}, \cite{lugosi1996}, and \cite{Kobos2011} and \cite{796369}.

Due to the common obstacles in general density estimation methods, besides the classification performance, the computational cost and the ratio of sample-size to number of the variable dimensions are two major issues discussed in the literature. 
The cons and pros of this type of nonparametric methods are intensively discussed in 
\cite{Anil2006} and
the a least-squares probabilistic classifier (LSPC) in \cite{Sugiyama10a} is a typical example.
The needs of the modern applications make the model-interpretation an important feature of classification methods. 
Therefore, the ability of the interpretation of nonparametric classification methods recently catch the most of the attention; see \cite{Yang2010} for example.  

In this study, we propose a new algorithm which effectively combine some conventional methods to achieve better performance than its predecessors at only little cost of the computational time, while retains certain model-interpretation ability. We illustrate the proposed method with simulation studies and real examples, and discuss its statistical properties as well.  

\section{Methodology}
\label{sec:2}
We consider a multiple-class classification problem in this study.
Suppose that we observe a $d$ dimensional feature vector for each subject,  denoted with $\mathcal{X} \subset \mathbb{R}^d$
 and  let $\mathcal{Y} = \{1, \dots , c\}$ be the set of class labels, where $c$ is the number of classes.
Assume that $\mathcal{X} \times \mathcal{Y}$ follows a joint probability density $p(\mathbf{x}, y)$. 
Then for a given $n$ paired samples of input $\mathbf{x}$ and output $y$:
$$\{(\mathbf{x}_i, y_i) \in \mathcal{X} \times \mathcal{Y}\}_{i = 1}^n.$$
Let $p(\mathbf{x})$ denote the marginal density of $\mathbf{x}$, then it follows that the conditional probability $p(y|\mathbf{x})$ is  
$$p(y|\mathbf{x}) = \frac{p(\mathbf{x}, y)}{p(\mathbf{x})}.$$ 
If both $p(\mathbf{x})$ and $p(\mathbf{x}, y)$ are known, then we can classify a test sample $\mathbf{x}$ to the class 
$\hat{y}$ with confidence $p(\hat{y}|\mathbf{x})$:
$$\hat{y} := \argmax_{y \in \mathcal{Y}} p(y|\mathbf{x}).$$
This suggests that we can build a classification rule based on some density estimation methods to estimate  $p(\mathbf{x})$ and $p(\mathbf{x}, y)$.
However, this is a challenging problem. When the number of classes $c$ is large and/or the data domain $\mathcal{X}$ is in high-dimensional setting,
it is time-consuming computation and not easy to have a classification rule with satisfactory accuracy.

\subsection{Learn simultaneously, compute separately}
\label{sec:3}


For  computation efficiency and numerically stability, we adopt the thought in the probabilistic classification model proposed by Masashi \cite{Sugiyama10a},  called Least-Squares Probabilistic Classifier (LSPC), in which we will calculate the class-posterior probabilities of each class simultaneously by formulating each of them as linear combinations of joint basis functions of $\mathbf{x}$ and $y$: $\phi(\mathbf{x} , y)$. The model of the probability $p(y|\mathbf{x} )$ is then written as:
\begin{equation}
	p(y|\mathbf{x}; \bm{\alpha}) := \sum_{i = 1}^b \alpha_i\phi(\mathbf{x} , y) = \bm{\alpha}^T\bm{\phi}(\mathbf{x} , y),   
\end{equation}
where $\bm{\alpha} = (\alpha_1, \dots, \alpha_b)^T \in \mathbb{R}^b$ is a parameter vector to be learned from samples, and 
$\bm{\phi}(\mathbf{x}, y)  \in \mathbb{R}^b$ is a non-negative basis function vector
such that 
\begin{equation}
	\bm{\phi}(\mathbf{x}, y) \geqslant \mathbf{0}_b \text{ for all } (\mathbf{x}, y) \in \mathcal{X} \times \mathcal{Y}.  
\end{equation}
Choosing kernel model as the basis functions, we have that 
$$p(y|\mathbf{x}; \bm{\alpha}) = \sum_{y'=1}^c\sum_{i=1}^n \alpha_i^{y'}\mathcal{K}'(\mathbf{x}, \mathbf{x_i}, y, y'),$$
where $\mathcal{K}'$ is some kernel function and then there are $c \times n$ parameters in the parameter vector $\bm{\alpha}  = (\alpha_1^1, \dots, \alpha_n^1, \dots, \alpha_1^c, \dots, \alpha_n^c)^T \in \mathbb{R}^{c\times n}$. 
To simplify it further via separating the input $\mathbf{x}$ and output $y$ with a 
kernel $\mathcal{K}$ for $\mathbf{x}$ and the delta kernel for $y$, 
we have 
\begin{equation}
	p(y|\mathbf{x}; \bm{\alpha}) = \sum_{y'=1}^c\sum_{i=1}^n \alpha_i^{y'}\mathcal{K}(\mathbf{x}, \mathbf{x_i}) \delta_{y, y'},
\label{eq:kernel}
\end{equation}
where $\delta_{y, y'}$ is the \textit{Kronecker delta}:
\begin{equation}
\label{eq:kronecker}
\delta_{y, y'} =
  \begin{cases}
    1       & \quad \text{if } y = y',\\
    0  & \quad \text{otherwise. }
  \end{cases}
\end{equation}
 
For a specific class $y$, the above model \eqref{eq:kernel} becomes
\begin{equation}
	p(y|\mathbf{x}; \bm{\alpha}) = \sum_{i=1}^n \alpha_i^{y}\mathcal{K}(\mathbf{x}, \mathbf{x_i}), \: y=1,\dots, c.
\end{equation}
In this case, the posterior probability model is in the form similar to the kernel density estimator:
\begin{equation}
\hat f_H(\mathbf{x}) = \frac{1}{n \det(H)}\sum_{i=1}^n\mathcal{K}(H^{-1}(\mathbf{x} - \mathbf{x}_i)),
\end{equation}
where $\mathcal{K}$ is the kernel function, $H = diag(h_1, \dots, h_d)$ is a diagonal matrix with bandwidths $h_1, \dots, h_d$, and $\det(H)$ is the determine of the matrix $H$ and also is the product of the bandwidths: $\prod_{j=1}^d h_j$. Then, assume $\mathcal{K}$ is a product Gaussian kernel and use the inverse of the product of bandwidths as the linear combination coefficients, the posterior probability function can be defined as

\begin{align}
	p(y|\mathbf{x}; H_y) & := \frac{1}{n}\sum_{i=1}^n \frac{1}{\det(H_y)}\mathcal{K}(H_y^{-1}(\mathbf{x} - \mathbf{x}_i))\nonumber\\
	& := \frac{1}{n}\sum_{i=1}^n \prod_{j=1}^d\frac{1}{h_j^y}K\Big(\frac{x_j - x_{ij}}{h_j^y}\Big),
\end{align}
where $H_y = diag(h_1^y, \dots, h_d^y)$ is a diagonal matrix with $h_j^y$ is the bandwidth in the $j$th coordinate for the class $y$.

Because the posterior function $p(y|\mathbf{x}; H_y)$ has higher probabilities in the regions where samples in class $y$ are dense; conversely, $p(y|\mathbf{x};H_y)$ has lower values in the regions where samples in class $y$ are sparse. 
When using the Gaussian kernel function to approximate a non-negative function, more kernels are needed in the region where the output of the target function is large. 
In this case, the kernels located in the trainings samples in class $y$ are the good choice.   
Hence, we reduce the number of kernels further by locating the kernels only at samples belongs to the target class and rewrite the posterior probability function as
\begin{align}
	p(y|\mathbf{x}; H_y) & =: \frac{1}{n_y}\sum_{i=1}^{n_y} \frac{1}{\det(H_{y})}\mathcal{K}(H_y^{-1}(\mathbf{x} - \mathbf{x}_i^y))   \notag    \\
	& =  \frac{1}{n_y}\sum_{i=1}^{n_y} \prod_{j=1}^d\frac{1}{h_j^y}K\Big(\frac{x_j - x_{ij}^y}{h_j^y}\Big)\label{eq:postprob}
\end{align}
where $n_y$ is the number of training samples in the specific class $y$. 
That is, instead of using the whole training dataset, the posterior probability function for the class $y$ is estimated just by the training input samples in class $y$: $\{\mathbf{x}_i^y = (x_{i1}^y, \dots, x_{id}^y)\}_{i=1}^{n_y}$, which have the most information and contribution. 
In this case, the posterior density function is just a kernel density function. Therefore, in order to prevent the confusion between the posterior density function and the posterior probability function, instead of $p(y|\mathbf{x}; H_y)$, we use $p_y(\mathbf{x}; H_y)$ as the notation of the posterior density function.

Let $\hat H_y = \textit{diag}(\hat h_1^y, \dots, \hat h_d^y)$ the estimated kernel bandwidths and substitute for $H_y$ in \eqref{eq:postprob},  then for $y = 1, \dots, c$ we have the posterior probability 
\begin{align} 
\hat{p}(y|\mathbf{x};\hat H_1,\dots, \hat H_c) &=  \frac{\frac{1}{n_y}\sum_{i=1}^{n_y} \prod_{j=1}^d\frac{1}{\hat h_j^y}K\Big(\frac{x_j - x_{ij}^y}{\hat h_j^y}\Big)} {\sum_{y'=1}^c \frac{1}{n_{y'}}\sum_{i=1}^{n_{y'}} \prod_{j=1}^d\frac{1}{\hat h_j^{y'}}K\Big(\frac{x_j - x_{ij}^{y'}}{\hat h_j^{y'}}\Big)} \notag \\
& = \frac{\hat p_y(\mathbf{x};\hat H_y)}{\sum_{y'=1}^c\hat p(y'|\mathbf{x};\hat H_{y'})}.
\label{simplified}
\end{align}
Using \eqref{simplified}, the original learning problem can be decomposed into independent class-wise learning problems; that is, we respectively estimate posterior density functions for each class using disjoint training samples.  Hence,  this method can  notably reduce the computational cost.

The LSPC algorithm will determine its kernel bandwidth parameters through minimizing the squared error of the posterior-probabilities with the quadratic regularizer.
Because one can analytically compute these bandwidth parameters through solving a  linear equation system such that this procedure is highly efficient,
The computational complexity of the simplified model,  as in \eqref{simplified},  drops from  the original $\mathcal{O}(c^3n^3)$ to $\mathcal{O}(c^{-1}n^3)$. 
However, in \cite{fan2008} and \cite{shao2011}, authors pointed out that a classification rule could perform as bad as random guessing without considering the sparsity condition and complex structure in high-dimensional data sets.
%
%
Thus, we target at constructing a density-estimation based classification rule with features of local bandwidth selection and variable selection, simultaneously.

\subsection{Sparse, greedy nonparametric kernel density estimation}
\label{sec:4}

Collecting large sized datasets is feasible,  while to analysis them becomes a crucial challenge.
The classification will not perform well or even break down, when regardless of sparsity and overfitting issues in analyzing these data sets via simply learning with all variables.
%
We can find a lot of discussions in the literature about the impact of the dimensionality on classification and authors of \cite{fan2008} pointed out that the difficulty of high-dimensional classification is intrinsically caused by the existence of many noise features that do not contribute to the reduction of classification error. 
In addition, the large amount of variables in such a data set does not usually offer additional benefits for decision making and may cause complexity and confusion in model-interpretation instead. 
Although the accuracy is a primary index for assessing classification performance,  practitioners can always benefit from further understanding the ``mechanism'' of a classification rule, which provide information beyond classification accuracy.  Thus, an effective method to reduce dimensionality and remove irrelevant data will be a key to increase learning accuracy.



Nonparametric density estimation estimates the density directly from the data without assuming a particular form for the underlying distribution, which offer a advantage to 
 a  greater flexibility in modeling a given dataset with less model-specification bias than that in the common parametric approaches.


The kernel density estimation is a popular nonparametric method  (
\cite{rosenblatt1956} and 
\cite{parzen1962}). 
A density estimator $\hat f_h(x)$ using a kernel function with a bandwidth $h$, $K$=$K_h(u) = K(u/h)/h$ is defined as
\begin{equation}
\hat f_h(x) = n^{-1}\sum_{i=1}^nK_h(x_i-x).
\end{equation} 
Assume that $K$ is symmetric, i.e. $K(u) = K(-u) $ and
\begin{align}
	\int K(u)du = 1.
\end{align}
It implies that the estimate at given $x$ is a weighted average, according to the kernel function $K_h$, of the probability mass of observed $x_i$s around it. 

The bandwidth parameter $h$ is also called smoothing parameter, which determines the ``width'' of a kernel function;
a large $h$ may over-smooth the density estimator and mask the structure of the data, while a small $h$ may make it spiky and hard to be described.


The cross-validation is a popular approach for bandwidth selection, in which one can estimate $h$ by minimizing the integrated squared error 
\cite{CIS-47810},  
\cite{BOWMAN1984}. 
However, there is a lack of stability in such an approach 
\cite{oro28198}.
Thus, many authors study proposed some modified methods to stabilize the bandwidth selection of cross-validation methods; see
\cite{chiu1991a} 
\cite{chiu1991b}  
\cite{chiu1992} 
\cite{Hall1992}.
Another useful approach is plug-in methods that try to minimize the mean integrated squared error to find the bandwidths 
was discussed in \cite{botev2010} and \cite{Silverman86}. 


The data-driven properties of kernel methods provide a flexible data modeling approach, however these methods usually suffer from the curse of dimensionality,  which is often in real-world tasks. The computational cost is one of those issues because we need to decide bandwidths of each dimension. 
In fact, Scott and Sain (2005) claim that the direct estimation of the full density by kernel methods is feasible in as many as six dimensions \cite{SCOTT2005229}. 

There are more approaches about kernel density function estimations in high-dimensional spaces in the literature 
\cite{LEIVAMURILLO2012}, 
\cite{Gu2013}. 
In these articles,  authors find the bandwidths using different criterions or objective functions and still ignore the affection of the redundant variables which has no or little impact to the estimate. 
For classification problems, adaptive estimates of each density function based on the individual training samples of each group with LSPC is a promising approach to improve the classification performance. Furthermore, we can learn the relevant features for density estimation of different classes. 
Hence, we adopt the thought of the greedily bandwidth selection using the regularization of derivative expectation operator (Rodeo) proposes in 
\cite{pmlr-v2-liu07a} in our estimates such that under the sparsity assumption, we can determine the relevant features with faster convergence rate and lower computational cost.

\subsection{Bandwidths selection}
\label{sec:5}
The proposed classification rule is based on the estimated posterior density.  Suppose that these  high impacts variables to the posterior density estimates of each subgroups are the only relevant features its corresponding subgroup. If we can identify the corresponding variable sets of each class, this information will give us the about features that ``describe'' each class and this information will largely improve the interpretation ability the proposed nonparametric density estimation based classification rule.

For class $y$, $y=1, \ldots, c$,  let $R_y$ be the index set in which $\mathbf{x}_{R_y} = \{x_j: j \in R_y\}$ is a set of variables which have high impacts to the posterior density of subgroup $y$.
Without lost of generality, we can rearrange the order of variables in $\mathbf{x}_{R_y}$ for each class $y$ such that
$j$ in $R_y$, for $1 \leqslant j \leqslant r_y$, are the high impact variables and  $x_j$ in $\mathbf{x}_{R_y^c}$ correspond to $r_y+1 \leqslant j \leqslant d$ are the rest of $d-r_y$ variables. Please note we use the notation $r_y$ here. Because the high impact subsets $R_y$ for each class could be different, the size $r_y$ of the subsets may vary among these subgroups.
We will drop the subscript $y$ of $R_y$ and $r_y$ below when there is no ambiguity and for simplification. 
It follows that we can rewrite the posterior density function as
\begin{align}
	p_y(\mathbf{x} = (x_1, \dots, x_d), H_y) &= g_y(\mathbf{x}_R, H_y^R)u(\mathbf{x}_{R^c})\\
	&= g_y(\mathbf{x}_R, H_y^R)\nonumber,
\end{align} 
where $u$ is an uniform function, $g_y$ is an unknown function depending only on the set $\mathbf{x}_R$ and $H_y^R = diag(h_1^y, \dots, h_r^y)$ is a $r\times r$ submatrix of $H_y$. 
If these $r$ variables are sufficient to estimate the density model of a class, and the others variables have little impact to the model, then one can exploit this fact such that the nonparametric estimates can convergence faster.
We employ this thinking  and rewrite \eqref{eq:postprob}
as follows: 
\begin{align}
p_y(\mathbf{x}; H_y) &=  \frac{1}{n_y}\sum_{i=1}^{n_y} \prod_{j=1}^d\frac{1}{h_j^y}K\Big(\frac{x_j - x_{ij}^y}{h_j^y}\Big) \notag \\
&=\frac{1}{n_y}\sum_{i=1}^{n_y} \left[\prod_{j=1}^r\frac{1}{h_j^y}K\Big(\frac{x_j - x_{ij}^y}{h_j^y}\Big)\right] \left[\prod_{j=r+1}^d\frac{1}{h_j^y}K\Big(\frac{x_j - x_{ij}^y}{h_j^y}\Big)\right]  \label{uniform}\\
& \approx \frac{1}{n_y}\sum_{i=1}^{n_y} \left[\prod_{j=1}^r\frac{1}{h_j^y}K\Big(\frac{x_j - x_{ij}^y}{h_j^y}\Big)\right]. \label{approx}
\end{align}
Equation \eqref{uniform}  is a product of the kernels of the relevant variables $\mathbf{x}_R$ and kernels of the irrelevant variables $\mathbf{x}_{R^c}$.  By assumption, the second term of \eqref{uniform} follows a uniform distribution and hence we have  \eqref{approx}.  
It follows that we can use a large bandwidths value on $h_j, j = r+1, \dots, d$ to obtain a smooth kernel density function for estimating such a uniform function,
and the greedily bandwidth selection approach will be useful in this case.
Thus, variable $x_j$ associates with a small value of bandwidth $h_j^y$ is relatively important in estimating this density. 
On the contrary,  it suggests that the variable $x_j$ may be irrelevant in the density model, if the derivatives $|Z_j|$ is small while the corresponding value of  $h_j^y$ is relatively large.  This fact suggests us a way to  find out the relative importances of variables to a particular model.
Because each class has its own set of important variables, this kind of information can help us to ``describe a class,'' which is essential in many practical applications.

Let  $\mathbf{x}=(x_i, \dots, x_d)^T$ be a $d$-dimensional point from  class $y$, then the estimate of the posterior density of $\mathbf{x}$ based on a kernel method is
\begin{equation*}
\hat p_y(\mathbf{x}; H_y) =  \frac{1}{n_y}\sum_{i=1}^{n_y} \prod_{j=1}^d\frac{1}{h_j^y}K\Big(\frac{x_j - x_{ij}^y}{h_j^y}\Big).
\end{equation*}
Algorithm Rodeo starts with a bandwidth matrix $H_y = \textit{diag}(h_0, \dots, h_0)$ with a large $h_0$,  and  then for $ 1 \leqslant j \leqslant d$,  computes derivatives
\begin{align}
Z_j &= \frac{\partial \hat p_y(\mathbf{x}; H_y)}{\partial h_j^y}\nonumber\\
	&= \frac{1}{n_y}\sum_{i=1}^{n_y} \frac{\partial}{\partial h_j^y}\Large\left[\prod_{k=1}^d\frac{1}{h_k^y}K\Big(\frac{x_k - x_{ik}}{h_k^y}\Big)\Large\right]\nonumber\\ 
	&\equiv \frac{1}{n_y}\sum_{i=1}^{n_y} Z_{ji}.
\end{align}
If $K$ is the Gaussian kernel, the $Z_j$ becomes
\begin{align}
Z_j &= \frac{1}{n_y}\sum_{i=1}^{n_y} Z_{ji}
	= \frac{1}{n_y}\sum_{i=1}^{n_y}\frac{(x_j - x_{ij})^2-(h_j^y)^2}{(h_j^y)^3} \prod_{k=1}^d\frac{1}{h_k^y}K\Big(\frac{x_k - x_{ik}}{h_k^y}\Big).
\end{align}
%

If $|Z_j|$ is large and changing $h_j^y$ leads to a substantial change in its corresponding estimate, then we prefer a smaller bandwidth,  $\beta \times h_j^y$ with some $\beta \in (0, 1)$ to the original $h_j^y$.  We repeat this process for each $j$ and keep shrinking its corresponding bandwidth in discrete steps $1, \beta, \beta^2, \dots$, until the value of $|Z_j|$ is less than a threshold $\lambda_j$.
To implement the test statistic $Z_j$, we compare it to its variance
\begin{equation}
	\sigma_j^2 = \mbox{Var}(Z_j) = \mbox{Var}(\frac{1}{n_y}\sum_{i=1}^{n_y} Z_{ji}) = \frac{1}{n_y}\mbox{Var}(Z_{j1}).
\end{equation}
The variance $\sigma^2_j$ is estimated by $s_j^2 = v_j^2/n_y$ where $v_j^2$ is the sample variance of the $Z_{ji}$s. Then follow the suggestion in \cite{pmlr-v2-liu07a}, we set the threshold $\lambda_j = s_j\sqrt{2\log(n_y c)}$, where $c = O(\log n_y)$, due to the trade-off between variance and bias.


For other kernel functions, we can still use this method to determine  $\lambda_j$ if the sizes of each class is large enough.
Algorithm~\ref{alg1} states,  given  a datapoint $\mathbf{x}$, how we use the Rodeo algorithm for the posterior density estimate with bandwidths selection in each subgroup $y$.

\begin{algorithm}[H]
\SetAlgoLined
  \caption{Rodeo for Posterior Density Estimation in subgroup $y$}
  \label{alg1}
    \KwData{
	\begin{itemize}
		\item	$\mathbf{x}_i=(x_{i1}, \dots, x_{id})^T, 
      i = 1, \dots, n_y$: training data set of subgroup $y$
      \item $\mathbf{x}$: a point on which we want to find the posterior density estimator
	\end{itemize}

	}
    \KwIn{
	\begin{itemize}
	\item $0 < \beta < 1$: reduce rate for bandwidth
	\item $h_0 = c_0/\log\log n_y$: initial bandwidth for some constant $c_0$
	\item $c_n = O(\log n_y)$
	\end{itemize}
    }
    \KwOut {
	\begin{itemize}
	\item Bandwidths $\hat H_y = \text{diag}(\hat h_1^y, \dots, \hat h_d^y)$
	\item Posterior density estimator: $\hat p_y(\mathbf{x};\hat H_y)$
	\end{itemize}
    }
    \vspace{0.1in}
    
    \vspace{0.1in}
    \bf{Initialization}\\

    $h_j^y= h_0, j = 1, \dots, d$\\
    $\mathcal{A} = \{1, 2, \dots, d\}$\\
    \vspace{0.1in}
    \While{$\mathcal{A}$ is nonempty}{
    	\For{$j \in \mathcal{A}$}{
		Estimate the derivative $Z_j$ and sample variance $s_j^2$.\\
		Compute the threshold $\lambda_j = s_j\sqrt{2\log(n_y c_n)}$.\\
		If $|Z_j| > \lambda_j$, set $h_{j,i}^y \leftarrow \beta h_{j,i}^y$; otherwise remove $j$ from  $\mathcal{A}$.
	}
    }

\end{algorithm}

\subsection{Feature selection}

Because these selections of bandwidths are data-dependent, 
we apply a statistical hypothesis testing method to decide whether there is significant differences among the bandwidths of each variable.
After selecting local bandwidths for each training data point $i, i = 1 \dots, n_y$ in subgroup $y$, 
we calculate $z$-scores of the mean bandwidths so they have mean 0 and are scaled to have standard deviation 1:
\begin{equation}
z_j^y = \frac{\bar h_j^y-mean(\bar h_1^y, \dots, \bar h_d^y)}{std(\bar h_1^y, \dots, \bar h_d^y)}, \: j= 1,\dots, d,
\label{Eq:zscore}
\end{equation}
where $\bar h_j^y = mean(h_{j,1}^y, \dots, h_{j,n_y}^y)$ and $h_{j,i}^y$ denote the selected bandwidth for the $i$th training datapoint.

%

If $x_j$ is a relevant variable, we expect a smaller selected bandwidth compared to that of an irrelevant one. Therefore, compare $z_j^y$ with a given cutpoint $\tau_0$, if it is smaller than the cutpoint, we think the corresponding variable $x_j$ is relatively important  and then include this variable  in $\mathbf{x}_R$.  It means
\begin{equation}
x_j \in 
\begin{cases}
    \mathbf{x}_R,& \text{if } \ z_j^y \leqslant \tau_0\\
    \mathbf{x}_{R^c},              & \text{otherwise}.
\end{cases}
\label{Eq:xr}
\end{equation}
The process of feature selection in subgroup $y$ is described in Algorithm~\ref{alg2}.
%

\begin{algorithm}[H]
\SetAlgoLined
  \caption{Feature Selection in subgroup $y$}
  \label{alg2}
    \KwData{
	\begin{itemize}
		\item	$\mathbf{x}_i=(x_{i1}, \dots, x_{id})^T, 
      i = 1, \dots, n_y$: training data set of subgroup $y$
	\end{itemize}
	}
    \KwIn{
	\begin{itemize}
	\item $\tau_0$: cutpoint for feature selection
	\end{itemize}
    }
    \KwOut {
	\begin{itemize}
	\item High impact set $\mathbf{x}_R$
	\end{itemize}
    }
    \vspace{0.1in}
    
    \emph{\textbf{Learning}}
        \begin{enumerate}
    	\item \For{$i = 1, \dots, n_y$}{
	Find the local bandwidths $\hat H_{y,i} = \text{diag}(\hat h_{1,i}^y, \dots, \hat h_{d,i}^y)$
	for data point $\mathbf{x}_i$ by Algorithm~\ref{alg1}
	}
	\item Calculate mean bandwidths $\bar h_j^y = mean(\hat h_{j,1}^y, \dots, \hat h_{j,n_y}^y)$, $j = 1, \dots, d$.
	\item Calculate $z$-scores $z_j^y$ of mean bandwidths defined in Eq. \eqref{Eq:zscore}.
	\item Decide $\mathbf{x}_R$ by Eq. \eqref{Eq:xr}.
    \end{enumerate}
\end{algorithm}

\section{Numerical Results}
\label{sec:7}
In this section, we demonstrate the proposed algorithm on both synthetic and real dataset. The accuracy is used to evaluate the algorithm's classification performance quantitatively. $m$ evaluation points are chosen randomly and evenly from $c$ classes and the predicted classes based on the proposed algorithm, which is shown in Algorithm~\ref{alg3}, are compared to the true class. 
The classification performance and the computation cost are compared with the results of LSPC. 
The default parameters are $c_0=1$, $c_n = \log n_y$, $\beta = 0.9$, and $\tau_0 = -1$. 

\begin{algorithm}[H]
\SetAlgoLined
  \caption{Classification with Rodeo}
  \label{C+R}
  \label{alg3}
    \KwData{
	\begin{itemize}
	\item $\mathbf{S}_y=\{\mathbf{x}^y_1, \dots, \mathbf{x}^y_{n_y}\}$, $y= 1, \dots, c$: $c$ dataset from c subgroups
	\item $\mathbf{x}_i$ : $i = 1, \dots, m$: testing data with class unknown\\
	\end{itemize}
	}
    \KwIn{
	\begin{itemize}
	\item $0 < \beta < 1$: reduce rate for bandwidth
	\item $h_0^y = c_0/\log\log n_y$: initial bandwidth for subgroup $y$, $y = 1, \dots, c$
	\end{itemize}
    }
    \KwOut {
	\begin{itemize}
	\item Estimated label: $\hat y_i$, $i = 1, \dots, m$
	\item Accuracy
	\end{itemize}
    }
    \vspace{0.1in}

    \emph{\textbf{Learning}}\\    
    \For{$y=1, \dots, c$}{
    	Use training data set $\mathbf{S}_y$
    	\For{$i = 1, \dots, m$}{
		Find the posterior density estimator $\hat p_y(\mathbf{x}_i; \hat H_{y,i})$ with selected bandwidths $\hat H_{y,i} = (\hat h_{1,i}^y, \dots, \hat h_{d,i}^y)$ by Algorithm~\ref{alg1}	}
    }
    
    \vspace{0.1in}
    \emph{\textbf{Classification}}\\
    	$\hat y_i = \argmax_y \frac{\hat p_y(\mathbf{x}_i;\hat H_{y,i})}{\sum_{y'=1}^c \hat p_{y'}(\mathbf{x}_i;\hat H_{y',i})}$, $i = 1, \dots, m$ \\
    Accuracy $= \sum_{i=1}^m \delta _{\hat y_i,  y_i}/m$, $\delta _{\hat y_i,  y_i}$ is the \textit{Kronecker delta} defined in Eq.\eqref{eq:kronecker}
    
    
\end{algorithm}

\subsection{Ten-Group Example}
Fist, we apply the algorithm on a  dataset with 10 groups. In each group, 6 out of the total 30 dimensions are chosen as relevant variables respectively. The relevant dimensions in the $y$-th group are: $\{y, y+1, \dots, y+5\}$.  The relevant dimensions in each group $y$ are generated as 
\begin{equation*}
X_i^y \sim \mathcal{N}(0.5, (0.02*(i-y+1))^2),\, \text{for } i = y : y+5,\, y = 1 : 10,
\end{equation*}
while the other irrelevant dimensions are generated as
\begin{equation*}
X_i^y \sim \text{Uniform}(0, 1), \,\text{for } i \in \{1:30\}\setminus \{y:y+5\},\, y = 1 : 10.
\end{equation*}
1000 examples were generated for each group from these distributions. The number of training and testing examples in each group are 150 and 100 respectively. Table \ref{Table1} displays the results based on 1000 simulations. The mean of the classification accuracy of the proposed algorithm is $67.49\%$,  which is far better than $21.11\%$ by LSPC. The cost of time duration is 80.96 sec v.s. 2.45 sec. It shows, with the proposed algorithm, even though the number of classes is large, learning through decomposed class-wise problems conquers the results by LSPC, which suffers by large number of classes. However, sacrificing time for accuracy is must and necessary.

Figure \ref{Figure1} illustrates the box-plots of the mean predicted bandwidths of testing samples which are assigned to each group for 1000 iterations. The plots show the bandwidths of the relevant dimensions in each group shrink towards zero, while the bandwidths of the irrelevant dimensions remain large. The mean of Z-scores are displayed on Table \ref{Table2}. The cell background of the relevant variables in each group are set as gray. It is clear that the Z-scores in all gray cells are negative, while others are positive. It tells the predicted bandwidths of relevant variables are relatively smaller than those of the other irrelevant variables. 
The smaller the bandwidth is, the more important the variable is. 
Based on these results, comparing the values of the predicted bandwidths is a good guidance to select the relevant variables. 
It reveals the proposed algorithm can find out the relevant variables for each groups respectively. That is the reason why the accuracy is highly improved.

\begin{table}
\begin{center}
\begin{tabular}{|c|c|c|}\hline
 & \multicolumn{2}{c|}{Method} \\ \hline 
 & new & LSPC \\ \hline \hline 
Accuracy &   0.6749(  0.0153) &   0.2111(  0.0143) \\ \hline 
Time &  80.9627(  3.5454) &   2.4544(  0.1639) \\ \hline 
\end{tabular}

\caption{ Classification results and computation cost of Ex1}
\label{Table1}
\end{center}
\end{table}    

\begin{figure}
	\includegraphics[scale=0.3]{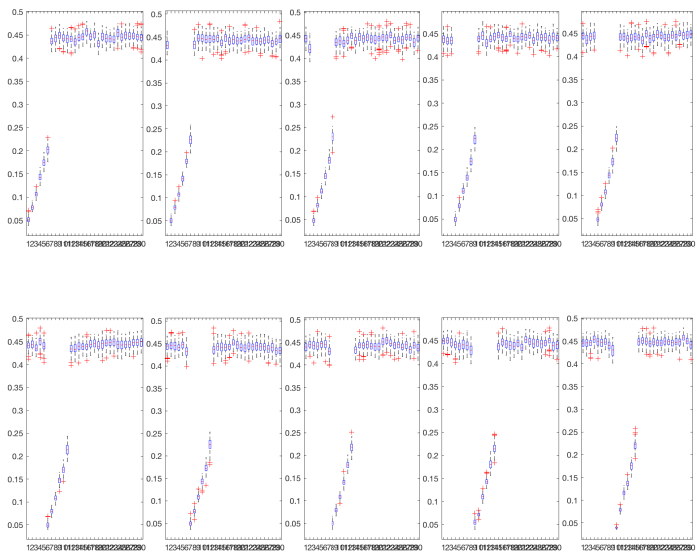}
	\caption{Ex1: The box-plots of mean predicted bandwidths for 10 groups.}
	\label{Figure1}
\end{figure}

\begin{table}    
\scalebox{0.65}{
\begin{tabular}{|c|cccccccccc|}\hline 
\diagbox{group~}{variable~~}& 1 & 2 & 3 & 4 & 5 & 6 & 7 & 8  & 9 & 10  \\ \hline \hline 
 1 &  -2.4945 \cellcolor[gray]{0.9} &  -2.2906 \cellcolor[gray]{0.9} &  -2.0713 \cellcolor[gray]{0.9} &  -1.7927 \cellcolor[gray]{0.9} &  -1.5563 \cellcolor[gray]{0.9} &  -1.3577 \cellcolor[gray]{0.9} &   0.4271 &   0.4578  &   0.5111 &   0.4665  \\ 
 2 &   0.4037 &  -2.5391 \cellcolor[gray]{0.9} &  -2.3148 \cellcolor[gray]{0.9} &  -2.1007 \cellcolor[gray]{0.9} &  -1.8287 \cellcolor[gray]{0.9} &  -1.5467 \cellcolor[gray]{0.9} &  -1.1630 \cellcolor[gray]{0.9}  &   0.4243  &   0.5217 &   0.5013  \\ 
 3 &   0.4876 &   0.3233 &  -2.5712 \cellcolor[gray]{0.9} &  -2.3153 \cellcolor[gray]{0.9} &  -2.0774 \cellcolor[gray]{0.9} &  -1.8116 \cellcolor[gray]{0.9} &  -1.5494 \cellcolor[gray]{0.9}  &  -1.1510 \cellcolor[gray]{0.9}  &   0.4432 &   0.4601  \\ 
 4 &   0.4834 &   0.4178 &   0.4339 &  -2.5299 \cellcolor[gray]{0.9} &  -2.3019 \cellcolor[gray]{0.9} &  -2.0567 \cellcolor[gray]{0.9} &  -1.8424 \cellcolor[gray]{0.9} &  -1.5743 \cellcolor[gray]{0.9}  &  -1.2160\cellcolor[gray]{0.9}  &   0.4722  \\ 
 5 &   0.4863 &   0.4270 &   0.4859 &   0.4842 &  -2.5408 \cellcolor[gray]{0.9} &  -2.2998 \cellcolor[gray]{0.9} &  -2.0880\cellcolor[gray]{0.9} &  -1.8080 \cellcolor[gray]{0.9}  &  -1.5877 \cellcolor[gray]{0.9} &  -1.1846 \cellcolor[gray]{0.9}  \\ 
 6 &   0.4666 &   0.4885 &   0.4371 &   0.5448 &   0.4636 &  -2.5231 \cellcolor[gray]{0.9} &  -2.2916\cellcolor[gray]{0.9} &  -2.0695 \cellcolor[gray]{0.9}  &  -1.7887 \cellcolor[gray]{0.9} &  -1.5981 \cellcolor[gray]{0.9}  \\ 
 7 &   0.4902 &   0.5029 &   0.5001 &   0.4827 &   0.5101 &   0.4120 &  -2.5386\cellcolor[gray]{0.9}  &  -2.3223 \cellcolor[gray]{0.9}  &  -2.0684 \cellcolor[gray]{0.9} &  -1.8073 \cellcolor[gray]{0.9}  \\ 
 8 &   0.4579 &   0.4956 &   0.4842 &   0.4877 &   0.5099 &   0.5026 &   0.3737 &  -2.5238 \cellcolor[gray]{0.9}  &  -2.3038 \cellcolor[gray]{0.9} &  -2.0830 \cellcolor[gray]{0.9}  \\ 
 9 &   0.5325 &   0.5382 &   0.4906 &   0.4596 &   0.4339 &   0.4780 &   0.4619 &   0.3750 &  -2.4896 \cellcolor[gray]{0.9} &  -2.3692 \cellcolor[gray]{0.9}  \\ 
 10 &   0.4683 &   0.4777 &   0.4652 &   0.5203 &   0.4863 &   0.4780 &   0.5114 &   0.4352 &   0.3545 &  -2.5888 \cellcolor[gray]{0.9}  \\ 
\hline \hline\diagbox{group~}{variable~~} & 11 & 12 & 13 & 14 & 15 & 16 & 17 & 18  & 19 & 20  \\ 
\hline \hline 1 &    0.4570 &   0.4687 &   0.4126 &   0.4872 &   0.4917 &   0.5714 &   0.4831 &   0.5172 &   0.3807 &   0.4928  \\ 
 2 &    0.5082 &   0.5023 &   0.4878 &   0.5190 &   0.4256 &   0.5284 &   0.4787 &   0.4745 &   0.4598 &   0.4753  \\ 
 3 &    0.4027 &   0.4776 &   0.5280 &   0.4630 &   0.5194 &   0.4789 &   0.5126 &   0.4936 &   0.4726 &   0.4892  \\ 
 4 &    0.5187 &   0.3978 &   0.4643 &   0.4965 &   0.5283 &   0.4851 &   0.4653 &   0.5482 &   0.4693 &   0.4699  \\ 
 5 &    0.4592 &   0.4662 &   0.4347 &   0.4654 &   0.4697 &   0.4717 &   0.4366 &   0.5116 &   0.4442 &   0.4540  \\ 
 6 &  -1.2664 \cellcolor[gray]{0.9} &   0.4146 &   0.4125 &   0.4481 &   0.4421 &   0.4596 &   0.4949 &   0.5118 &   0.4619 &   0.5037 \\ 
 7 &  -1.5889 \cellcolor[gray]{0.9} &  -1.1806 \cellcolor[gray]{0.9} &   0.4432 &   0.4774 &   0.4817 &   0.4741 &   0.4894 &   0.5457 &   0.5200 &   0.4843 \\ 
 8 &  -1.8292 \cellcolor[gray]{0.9} &  -1.5358 \cellcolor[gray]{0.9} &  -1.2452 \cellcolor[gray]{0.9} &   0.4231 &   0.4856 &   0.4778 &   0.5055 &   0.4850 &   0.4709 &   0.4654 \\ 
 9 &  -2.0683 \cellcolor[gray]{0.9} &  -1.8201 \cellcolor[gray]{0.9} &  -1.5176 \cellcolor[gray]{0.9} &  -1.2497 \cellcolor[gray]{0.9} &   0.4399 &   0.5152 &   0.4765 &   0.4386 &   0.4365 &   0.4931 \\ 
 10 &  -2.2917 \cellcolor[gray]{0.9} &  -2.0137 \cellcolor[gray]{0.9} &  -1.8417 \cellcolor[gray]{0.9} &  -1.5614 \cellcolor[gray]{0.9} &  -1.2107 \cellcolor[gray]{0.9} &   0.4760 &   0.5137 &   0.5020 &   0.4660 &   0.4996 \\ 
\hline \hline\diagbox{group~}{variable~~} & 21 & 22 & 23 & 24 & 25 & 26 & 27 & 28  & 29 & 30  \\ 
\hline \hline1 &   0.4685 &   0.4576 &   0.4486 &   0.5553 &   0.5031 &   0.5277 &   0.4911 &   0.5012 &   0.4959 &   0.4890  \\ 
2 &   0.4959 &   0.5185 &   0.4571 &   0.4652 &   0.4583 &   0.4951 &   0.4958 &   0.4395 &   0.4531 &   0.5036  \\ 
3 &   0.5100 &   0.5044 &   0.5390 &   0.4594 &   0.4909 &   0.4587 &   0.5239 &   0.5081 &   0.4300 &   0.4996  \\ 
4 &   0.5000 &   0.4791 &   0.5184 &   0.4417 &   0.5033 &   0.4763 &   0.5127 &   0.4436 &   0.5056 &   0.4899  \\ 
5 &   0.5227 &   0.4967 &   0.4558 &   0.4721 &   0.5185 &   0.4748 &   0.5496 &   0.4946 &   0.5072 &   0.5199  \\ 
6 &   0.5003 &   0.5224 &   0.5226 &   0.4763 &   0.4625 &   0.4722 &   0.4916 &   0.5148 &   0.5122 &   0.5125  \\ 
7 &   0.4649 &   0.5012 &   0.4798 &   0.4846 &   0.4903 &   0.4816 &   0.4554 &   0.4725 &   0.4378 &   0.4244  \\ 
8 &   0.5506 &   0.5633 &   0.5117 &   0.4825 &   0.5017 &   0.4491 &   0.5006 &   0.4098 &   0.4774 &   0.4493  \\ 
9 &   0.4234 &   0.5536 &   0.5358 &   0.4868 &   0.4954 &   0.4824 &   0.5025 &   0.5429 &   0.4540 &   0.4681  \\ 
10 &   0.4763 &   0.4593 &   0.4922 &   0.4712 &   0.4606 &   0.4965 &   0.4942 &   0.5559 &   0.5055 &   0.4420  \\ 
\hline \hline
\end{tabular}

}
\caption{ Ex1: Z-score of the mean predicted bandwidths.}
\label{Table2}
\end{table}

\subsection{Example with Special Located Means between Groups}
In this example, a data set with 5 groups is generated. The first two variables out of the total 10 dimensions are chosen as the relevant variables in all groups. In this case, the last 8 variables are generated by uniform distribution. As for the two relevant variables, the means of the last four groups are scattered around the first group on purpose, which are displayed on Figure \ref{fig:Ex2}, and the true values are listed on Table \ref{Table3}.
The distributions of the two relevant variable in each group are are defined as
\begin{equation}
N\bigg(\begin{bmatrix}  \mu^y_1 \\ \mu^y_2 \end{bmatrix}, \begin{bmatrix} 0.1^2 & 0 \\ 0 & 0.2^2 \end{bmatrix}\bigg), y = 1:5
\end{equation}
In order to the see the influence of the number of groups on classification performance, experiments with 2, 3, 4 and 5 groups and all combinations are considered. The number of training examples and the number of testing examples in each group are 200 and 150 respectively. 

The results of 1000 simulations are displayed on Table \ref{Table4}. Different combination means different data structure between groups, therefore, even having the same number of groups in classification, accuracy rate varies between different combinations. It is expectable that the classification performance declines when number of groups goes up, because the data structure becomes complicated with much overlaps between group spaces. 
The decision bounds between groups become more and more vague. The highest accuracy of the proposed algorithm and LSPC are $0.9954$ and $0.9950$ in 2 groups , $0.9298$ and $0.9295$ in 3 groups , $0.8841$ and $0.8559$ in 4 groups  and $0.7769$ and $0.7100$ in 5 groups. Both accuracy rates decline as number of groups increase. However, LSPC has higher decrease range, about $28\%$, while the proposed algorithm is $22\%$, when the number of groups increases from 2 groups to 5 groups.  

On the other hand, look at the lowest accuracy rates among these different combinations of the same number of groups, the proposed algorithm seems remain on the same level, 0.7818, 0.7836, 0.7652, and 0.7769, while LSPC keeps falling, 0.7639, 0.7246, 0.7114, and 0.7100, when number of groups increases. It shows learning with decomposed class-wise problems on multi-class classification is more suitable and effective.
The mean of Z-scores of the mean predicted bandwidths are displayed on Table \ref{Table5}. In order to reduce the space, we just show the results of the first combination in each setting.  The first two variables are the relevant variables and the mean of Z-scores are all negative. Therefore, the variable selection results based on the Z-scores of the mean predicted bandwidths are reliable.
\vspace{1in}

\begin{table}
        \begin{center}

        \begin{tabular}{|c|rr|}\hline
\diagbox{group}{variable} & 1 & 2 \\ \hline \hline 
1 &   0.0000 &   0.0000 \\ \hline 
2 &   0.1635 &   0.2044 \\ \hline 
3 &  -0.2452 &   0.1431 \\ \hline 
4 &  -0.2180 &  -0.3815 \\ \hline 
5 &   0.3815 &  -0.1907 \\ \hline 
\end{tabular}

        \caption{ Ex2: Mean and standard deviation of the related variables in each group.}
        \label{Table3}
    \end{center}
\end{table}

\begin{table}    
        \begin{center}
	\scalebox{0.63}{
        \begin{tabular}{|c|c|c||c|c|c|}\hline 
\# of groups & \multicolumn{2}{|c|}{2} & \# of groups & \multicolumn{2}{|c|}{3} \\ \hline 
\diagbox{combinations~}{Method~~} & new & LSPC & \diagbox{combinations~}{Method~~} & new & LSPC \\ \hline 
        4        5 &   0.9937(  0.0043) &   0.9944(  0.0044) &        3        4        5 &   0.9298(  0.0118) &  0.9295(  0.0114) \\ \hline 
        3        5 &   0.9954(  0.0035) &   0.9950(  0.0037) &        2        4        5 &   0.9297(  0.0116) &  0.9165(  0.0135) \\ \hline 
        3        4 &   0.9000(  0.0154) &   0.9049(  0.0158) &        2        3        5 &   0.9152(  0.0137) &  0.8647(  0.0154) \\ \hline 
        2        5 &   0.9046(  0.0172) &   0.8940(  0.0214) &        2        3        4 &   0.9079(  0.0142) &  0.9013(  0.0143) \\ \hline 
        2        4 &   0.9875(  0.0063) &   0.9757(  0.0091) &        1        4        5 &   0.9102(  0.0133) &  0.8752(  0.0170) \\ \hline 
        2        3 &   0.9701(  0.0084) &   0.9645(  0.0093) &        1        3        5 &   0.9005(  0.0131) &  0.8002(  0.0210) \\ \hline 
        1        5 &   0.9658(  0.0104) &   0.9555(  0.0115) &        1        3        4 &   0.8266(  0.0178) &  0.8166(  0.0175) \\ \hline 
        1        4 &   0.9016(  0.0184) &   0.8977(  0.0164) &        1        2        5 &   0.7937(  0.0207) &  0.7758(  0.0191) \\ \hline 
        1        3 &   0.8845(  0.0186) &   0.8506(  0.0219) &        1        2        4 &   0.7883(  0.0186) &  0.7246(  0.0201) \\ \hline 
        1        2 &   0.7818(  0.0219) &   0.7639(  0.0248) &        1        2        3 &   0.7836(  0.0163) &  0.7335(  0.0174) \\ \hline 
\multicolumn{6}{|c|}{ } \\ \hline 
\# of groups & \multicolumn{2}{|c|}{4} & \# of groups & \multicolumn{2}{|c|}{5} \\ \hline 
\diagbox{combinations~}{Method~~} & new & LSPC & \diagbox{combinations~}{Method~~} & new & LSPC \\ \hline 
        1        2        3        4 &   0.7652(  0.0161) &   0.7128(  0.0158) &        1        2        3        4        5 &   0.7769(  0.0155) &  0.7100(  0.0136) \\ \hline 
       1        2        3        5 &   0.7952(  0.0159) &   0.7114(  0.0175) & & &\\ \hline 
       1        2        4        5 &   0.7986(  0.0146) &   0.7453(  0.0167) & & &\\ \hline 
       1        3        4        5 &   0.8544(  0.0154) &   0.7764(  0.0178) & & &\\ \hline 
       2        3        4        5 &   0.8841(  0.0122) &   0.8559(  0.0141) & & &\\ \hline 
\end{tabular}

        }
        \caption{ Ex2: Accuracy results of classification on different combinations and different number of groups; trails = 1000, $\#$ of training = 200, $\#$ of testing = 150.}
        \label{Table4}
        \end{center}
\end{table}    

%
%
%

\begin{table}    
        \begin{center}
	\scalebox{0.68}{
        \begin{tabular}{||c||>{\columncolor[gray]{0.9}}c>{\columncolor[gray]{0.9}}ccccccccc||}\hline 
 \diagbox{\# of groups~~}{variables} & 1 & 2 & 3 & 4 & 5 & 6 & 7 & 8 & 9 & 10 \\ \hline \hline 
 \multirow{2}{*}{2}  &  -2.5575 &  -0.9233 &   0.3909 &   0.3847 &   0.4351 &   0.3806 &   0.4577 &   0.4907  &   0.5380 &   0.4030 \\ &  -2.5553 &  -0.9317 &   0.4765 &   0.4531 &   0.4306 &   0.4949 &   0.3278 &   0.4768  &   0.4120 &   0.4153\\\hline\hline
 \multirow{3}{*}{3}  &  -2.5007 &  -1.0458 &   0.3606 &   0.5410 &   0.4224 &   0.4528 &   0.4586 &   0.4539  &   0.4376 &   0.4196 \\ &  -2.5230 &  -0.9986 &   0.4032 &   0.3796 &   0.4316 &   0.4084 &   0.4668 &   0.4861  &   0.5294 &   0.4165\\ &  -2.5604 &  -0.9191 &   0.4683 &   0.4481 &   0.4297 &   0.5118 &   0.3120 &   0.4782  &   0.4171 &   0.4143\\\hline \hline
 \multirow{4}{*}{4}  &  -2.5571 &  -0.9231 &   0.4801 &   0.3717 &   0.4291 &   0.4099 &   0.4213 &   0.5028  &   0.4293 &   0.4360 \\ &  -2.5165 &  -1.0022 &   0.3638 &   0.4267 &   0.3430 &   0.4214 &   0.4407 &   0.5711  &   0.4662 &   0.4859\\ &  -2.5067 &  -1.0346 &   0.3827 &   0.5194 &   0.4629 &   0.4501 &   0.4416 &   0.4563  &   0.4149 &   0.4135\\ &  -2.5174 &  -1.0105 &   0.3908 &   0.3786 &   0.4308 &   0.4238 &   0.4664 &   0.5046  &   0.5192 &   0.4137\\\hline \hline \multirow{5}{*}{5}  &  -2.5434 &  -0.9502 &   0.4740 &   0.3641 &   0.4186 &   0.4337 &   0.4355 &   0.4973  &   0.4396 &   0.4308 \\ &  -2.5109 &  -1.0177 &   0.3855 &   0.4493 &   0.3519 &   0.4211 &   0.4537 &   0.5652  &   0.4460 &   0.4558\\ &  -2.5146 &  -1.0198 &   0.3832 &   0.5262 &   0.4492 &   0.4503 &   0.4460 &   0.4503  &   0.4157 &   0.4135\\ &  -2.5127 &  -1.0161 &   0.3838 &   0.3785 &   0.4410 &   0.4059 &   0.4788 &   0.4973  &   0.5296 &   0.4139\\ &  -2.5492 &  -0.9417 &   0.4646 &   0.4365 &   0.4331 &   0.5065 &   0.3319 &   0.4756  &   0.4151 &   0.4275\\\hline \hline
\end{tabular}

        }
        \caption{ Ex2: The mean of Z-scores of the mean predicted bandwidth for all variables; trails = 1000.}
        \label{Table5}
    \end{center}
\end{table}    
%

\subsection{Examples with different number of training examples}
In this example, we use different number of training examples to see the influence on multi-class classification. 
The data set are generated as the data in Example 2. The number of training examples in each group are 50, 150, 500, and 1000. For each multi-class setting, we perform the experiment on the first combination. 

In Table \ref{Table6}, on both methods, the accuracy rates increase when the number of training examples increase. However, even when the number of training examples is just 50, the proposed algorithm has about $10\%$ higher performance than LSPC. In LSPC, the accuracy increases about $10\%$ to $15\%$ when training samples increases from 50 to 1000.
In other words, through the decomposed algorithm, the proposed method has a more stable results and can achieve better classification accuracy easily even when training samples are relatively small. The Z-scores of the predicted bandwidth for all variables are listed on Table \ref{Table8}. The values of the first two relevant variables are all negative. 
\begin{table}
    
        \begin{center}
	\scalebox{0.65}{

        \begin{tabular}{||c||cccc||cccc||}\hline 
 & \multicolumn{8}{c||}{Methods}  \\  
 & \multicolumn{4}{c}{new} & \multicolumn{4}{c||}{LSPC} \\ \hline \hline
\diagbox{\# of training~}{\# of groups~~} & 2 & 3 & 4 & 5 & 2 & 3 & 4 & 5 \\ \hline \hline
 \multirow{2}{*}{50}  &   0.7504 &   0.7424 &   0.7168 &   0.7392 &   0.6603 &   0.5987 &   0.5917 &   0.6215 \\ & (  0.0373)& (  0.0287)& (  0.0279)& (  0.0226)& (  0.0421)& (  0.0354)& (  0.0323)& (  0.0304)\\\hline 
 \multirow{2}{*}{150}  &   0.7706 &   0.7620 &   0.7486 &   0.7728 &   0.7396 &   0.7143 &   0.6899 &   0.7079 \\ & (  0.0316)& (  0.0243)& (  0.0210)& (  0.0217)& (  0.0325)& (  0.0245)& (  0.0203)& (  0.0184)\\\hline 
 \multirow{2}{*}{500}  &   0.7862 &   0.7811 &   0.7537 &   0.7833 &   0.7927 &   0.7532 &   0.7210 &   0.7342 \\ & (  0.0289)& (  0.0225)& (  0.0209)& (  0.0185)& (  0.0286)& (  0.0230)& (  0.0186)& (  0.0164)\\\hline 
 \multirow{2}{*}{1000}  &   0.7877 &   0.7779 &   0.7584 &   0.7853 &   0.8052 &   0.7527 &   0.7299 &   0.7485 \\ & (  0.0270)& (  0.0237)& (  0.0222)& (  0.0176)& (  0.0252)& (  0.0228)& (  0.0188)& (  0.0193)\\\hline 
\end{tabular}

        }
        \caption{ Ex3: Accuracy results of classification on different number of training examples; trails = 1000.}
        \label{Table6}
    \end{center}
\end{table}

%

\begin{table}
        \begin{center}
	\scalebox{0.65}{
        \begin{tabular}{||c||cccc||cccc||}\hline 
 & \multicolumn{8}{c||}{Methods}  \\  
 & \multicolumn{4}{c}{new} & \multicolumn{4}{c||}{LSPC} \\ \hline \hline
\diagbox{\# of training~}{\# of groups~~} & 2 & 3 & 4 & 5 & 2 & 3 & 4 & 5 \\ \hline \hline
 \multirow{2}{*}{50}  &   1.0809 &   2.5371 &   4.7232 &   7.5061 &   0.0536 &   0.0905 &   0.1334 &   0.1904 \\ & (  0.0462)& (  0.1132)& (  0.4426)& (  0.7774)& (  0.0069)& (  0.0138)& (  0.0277)& (  0.0529)\\\hline 
 \multirow{2}{*}{150}  &   1.3949 &   3.0299 &   5.5781 &   9.4078 &   0.2759 &   0.4617 &   0.7104 &   1.1005 \\ & (  0.0581)& (  0.0827)& (  0.2526)& (  0.9324)& (  0.0252)& (  0.0271)& (  0.0543)& (  0.1422)\\\hline 
 \multirow{2}{*}{500}  &   2.0706 &   4.5040 &   7.9757 &  13.8507 &   3.8471 &   6.9419 &  11.1803 &  17.4199 \\ & (  0.0302)& (  0.0529)& (  0.0940)& (  1.5575)& (  0.1800)& (  0.1914)& (  0.2178)& (  1.7608)\\\hline 
 \multirow{2}{*}{1000}  &   3.1061 &   6.7781 &  11.9751 &  18.8411 &  22.7555 &  39.8806 &  59.3064 &  87.6326 \\ & (  0.0169)& (  0.0236)& (  0.0358)& (  0.6141)& (  0.9013)& (  0.9520)& (  1.0279)& (  3.0381)\\\hline 
\end{tabular}

        }
        \caption{ Ex3: Time duration of classification on different number of training examples; trails = 1000.}
        \label{Table7}
    \end{center}
\end{table}    

\begin{table}
        \begin{center}
	\scalebox{0.65}{
        \begin{tabular}{||c||>{\columncolor[gray]{0.9}}c>{\columncolor[gray]{0.9}}ccccccccc||}\hline 
 \diagbox{\# of groups~~}{variables} & 1 & 2 & 3 & 4 & 5 & 6 & 7 & 8 & 9 & 10 \\ \hline \hline 
 \multirow{2}{*}{2}  &  -2.4735 &  -1.0303 &   0.4167 &   0.3886 &   0.4392 &   0.4435 &   0.4175 &   0.4126  &   0.5068 &   0.4788 \\ &  -2.4797 &  -1.0014 &   0.4337 &   0.4401 &   0.5089 &   0.4098 &   0.3679 &   0.3848  &   0.4546 &   0.4814\\\hline\hline
 \multirow{3}{*}{3}  &  -2.4710 &  -1.0087 &   0.4111 &   0.4043 &   0.4058 &   0.4414 &   0.4526 &   0.4024  &   0.4837 &   0.4785 \\ &  -2.4900 &  -0.9813 &   0.4289 &   0.4389 &   0.4920 &   0.4379 &   0.3411 &   0.3914  &   0.4740 &   0.4673\\ &  -2.5256 &  -0.9002 &   0.4057 &   0.4196 &   0.3936 &   0.4843 &   0.4025 &   0.4680  &   0.4846 &   0.3675\\\hline \hline
 \multirow{4}{*}{4}  &  -2.4595 &  -1.0292 &   0.4561 &   0.4062 &   0.4164 &   0.4492 &   0.4079 &   0.4270  &   0.4704 &   0.4555 \\ &  -2.4927 &  -0.9814 &   0.4359 &   0.4452 &   0.4879 &   0.4502 &   0.3707 &   0.3751  &   0.4729 &   0.4361\\ &  -2.4827 &  -0.9911 &   0.4436 &   0.4438 &   0.4179 &   0.5118 &   0.4158 &   0.4170  &   0.4383 &   0.3857\\ &  -2.4456 &  -1.0667 &   0.4406 &   0.5397 &   0.4126 &   0.4330 &   0.3372 &   0.4458  &   0.4434 &   0.4600\\\hline \hline \multirow{5}{*}{5}  &  -2.4514 &  -1.0243 &   0.3720 &   0.4073 &   0.4435 &   0.4213 &   0.4486 &   0.4050  &   0.5133 &   0.4646 \\ &  -2.4686 &  -1.0032 &   0.4369 &   0.4308 &   0.4808 &   0.4067 &   0.3495 &   0.4428  &   0.4513 &   0.4730\\ &  -2.4996 &  -0.9476 &   0.4403 &   0.3970 &   0.4123 &   0.5161 &   0.4088 &   0.4414  &   0.4283 &   0.4030\\ &  -2.4495 &  -1.0469 &   0.4674 &   0.5668 &   0.3666 &   0.4518 &   0.3489 &   0.4348  &   0.4346 &   0.4255\\ &  -2.5386 &  -0.8790 &   0.4842 &   0.3599 &   0.5038 &   0.4184 &   0.4397 &   0.4106  &   0.4020 &   0.3989\\\hline \hline
\end{tabular}

        }
        \caption{ Ex3: The mean of Z-scores of the mean predicted bandwidth for all variables; trails = 1000.}
        \label{Table8}
    \end{center}
\end{table}

%
%
%

\subsection{Anuran Species Classification}
This example uses the anuran calls dataset \cite{Dua:2017} for recognizing and making classification of anuran species. This dataset was created by segmenting 60 audio records belonging to 4 different families, 8 genus, and 10 species. Total 7195 syllables were identified from the 60 bioacoustic signals after segmenting. Then each syllable is represented by a set of features extracted by Mel-Frequency Spectral Coefficients (MFCCs), which perform a spectral analysis based on a triangular filter-bank logarithmically spaced in the frequency domain. Therefore, each instance in the data set is a feature set of MFCCS coefficients which belong to a special species. Here we focus on classification of the main 7 out of total 10 species: Leptodactylus fuscus, Adenomera andreae, Adenomera hylaedactyla, Hyla minuta, Hypsiboas cinerascens, Hypsiboas cordobae, and  Ameerega trivittata. Besides the original dataset, we extend the dataset by addding 5 noise attributes with mean 0 and variance 1 for being the unrelated variables. In each trail, we randomly select 100 and 50 examples from each species as the training data and the evaluation points respectively. 

The box-plots of mean predicted bandwidths with noise features for 100 trials are displayed in Figure \ref{fig:frogs2}. It is clear that the bandwidths of the added noise attributes(the last 5 attributes) remain large, while the other bandwidths shrink. The results is consistent with the condition that these 5 attributes are irrelevant. The classification performance with and without the noise attributes are both shown in Table \ref{tb:frogs1}. With about $3\%$ increase on accuracy and precision while the specificity are similar,  the main progress after removing the irrelevant variables is on the true positive rate, the ability of target identification.


\begin{figure}
\begin{center}
	\includegraphics[scale=0.65]{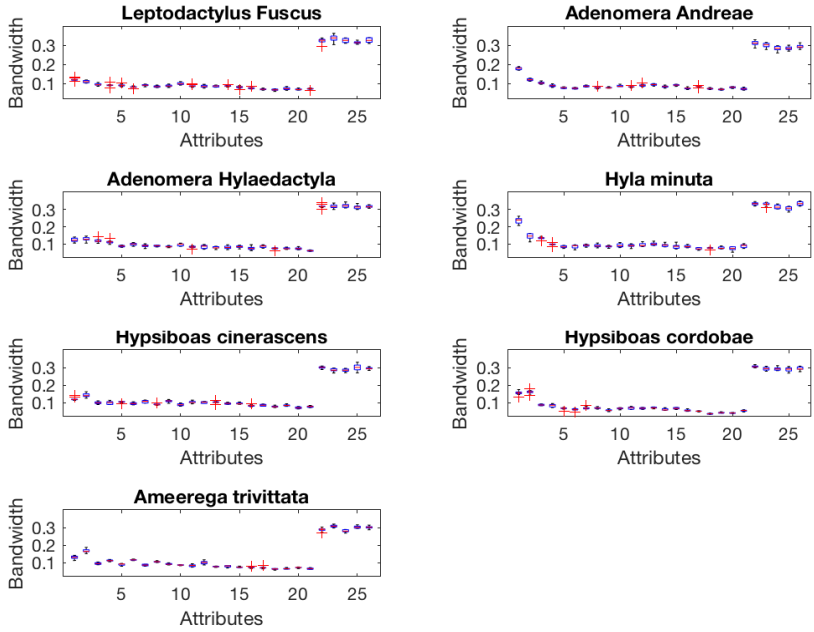}
	\caption{The box-plots of mean predicted bandwidths for anuran calls.}
	\label{fig:frogs2}
\end{center}
\end{figure}

\begin{table}
\begin{center}
\scalebox{0.9}{
\begin{tabular}{|l||r|r|r||r|r|r|}
\hline
     & \multicolumn{3} {c||}{Original dataset} & \multicolumn{3} {c|}{Extended dataset}\\
     & \multicolumn{3} {c||}{(without noise features)} & \multicolumn{3} {c|}{(with noise features)}\\ \hline

     & Accuracy & Precision & Specificity  & Accuracy & Precision & Specificity\\ \hline\hline
mean & 0.9155   & 0.9173      & 0.9859    & 0.8741   & 0.880731      & 0.9790\\
Std  & 0.0124   & 0.0019      & 0.0021   & 0.0142   & 0.0135      & 0.0024\\ \hline
\end{tabular}
}
\caption{Classification results of anuran species.}

\label{tb:frogs1}
\end{center}
\end{table}

\newpage
\subsection{Waveform dataset}
This example uses a generated waveform data \cite{Dua:2017}. 3 classes of waves are generated and each class is generated from a combination of 2 of 3 "base" waves. Each instance is generated by 21 related attributes with noise. After the generation of the waveform data, another 19 unrelated noise attributes with mean 0 and variance 1 are added to the dataset. The box-plots of bandwidth output of 3 classes for 100 trials are shown in Figure \ref{figure:wave}. The plots show that the bandwidths of the irrelevant attributes: 22:40, are larger than those of the relevant attributes: 1:21. The classification result 

\begin{figure}

	\includegraphics[scale=0.65]{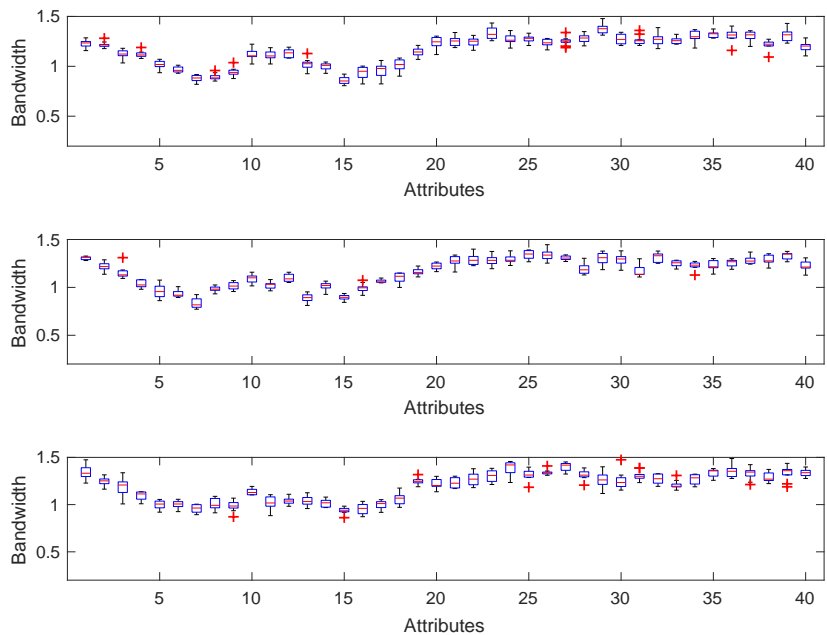}
		\caption{The box-plots of mean predicted bandwidths for the waveform dataset.}
	\label{figure:wave}

\end{figure}

\subsection{Handwritten digit dataset}
This example apply the algorithm on the handwritten digit data. In each digit dataset, 100 gray images were randomly selected for training and another 100 images are selected for testing. Each handwritten digit has $8\times 8 = 64$ pixels and we consider each pixel is a variable. Therefore, it is a 10-class classification through 64-dimensional density estimation problem. The classification results are shown in Table \ref{tb:digit1}. The performance is every great. Accuracy, specificity, and specificity are all great than $95\%$. The box-plots in Figure \ref{fig:digit1} and Figure \ref{fig:digit2} illustrate the mean of the selected bandwidths of the testing images for 100 trials. In each plot/digit, some bandwidths remain on the top with very small interquartile range(IQR) and some have very large interquartile range. Unlike the situation that the attributes with high bandwidths are the irrelevant variables, in fact, these attributes that have large IQR are the irrelevant variables in this example. Because instead of the uniform distribution, the background pixels of the image data have a density close to point mass. In this case,  the corresponding bandwidths could drop to a very small value sometimes. Therefore, in the box-plots of all 10 digits, we can see the bandwidths of attributes 1, 8,9,16,17,24,25,32,33,40, 41,48,49,56,57,64, which are the pixels on the top and bottom of the image, meet the situation. 
\begin{table}
\begin{center}
\begin{tabular}{|l|r|r|r|}
\hline
     & Accuracy & Precision & Specificity   \\ \hline\hline
mean & 0.9698   & 0.9966      & 0.9797  \\
Std  & 0.0047   & 0.0005      & 0.0030   \\ \hline
\end{tabular}
\caption{Classification results of handwritten digit dataset.}

\label{tb:digit1}
\end{center}
\end{table}

\begin{figure}

	\includegraphics[scale=0.6]{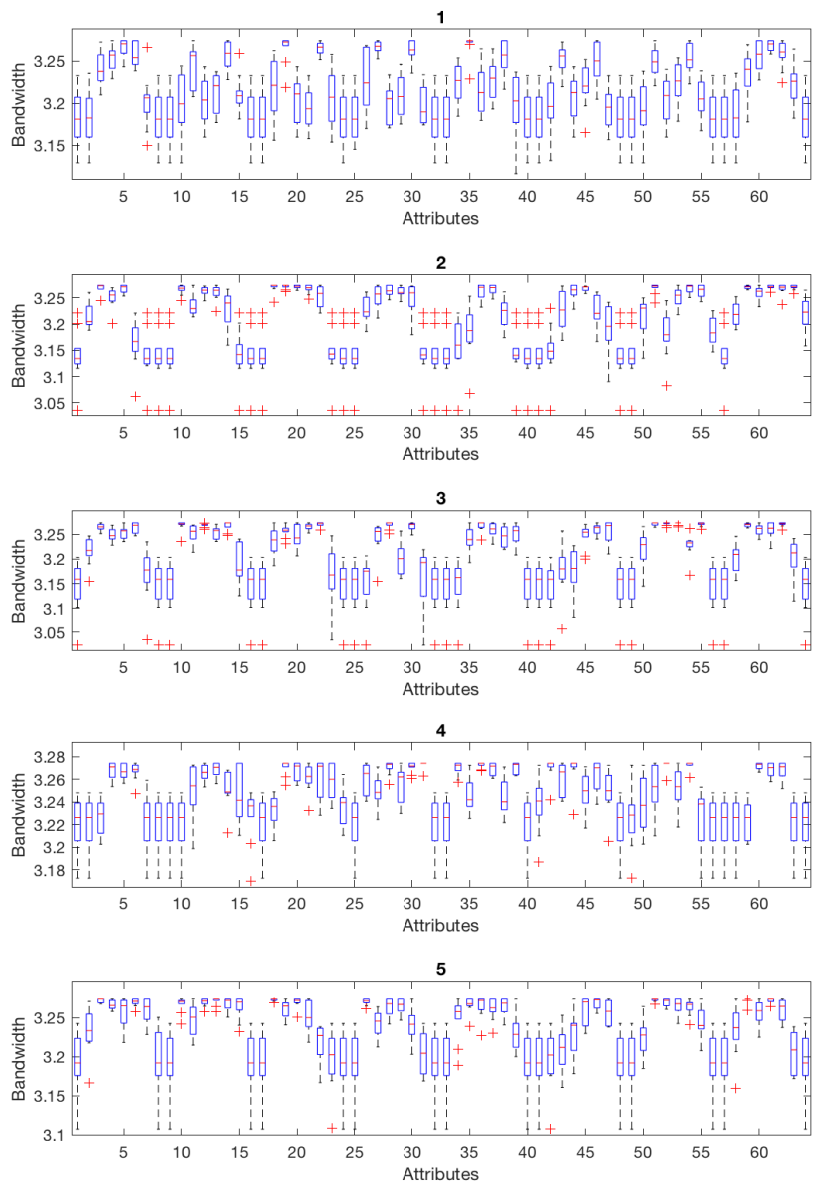}
	\caption{The box-plots of mean predicted bandwidths for handwritten digit 1, 2, 3, 4, 5.}
	\label{fig:digit1}

\end{figure}
\begin{figure}

	\includegraphics[scale=0.6]{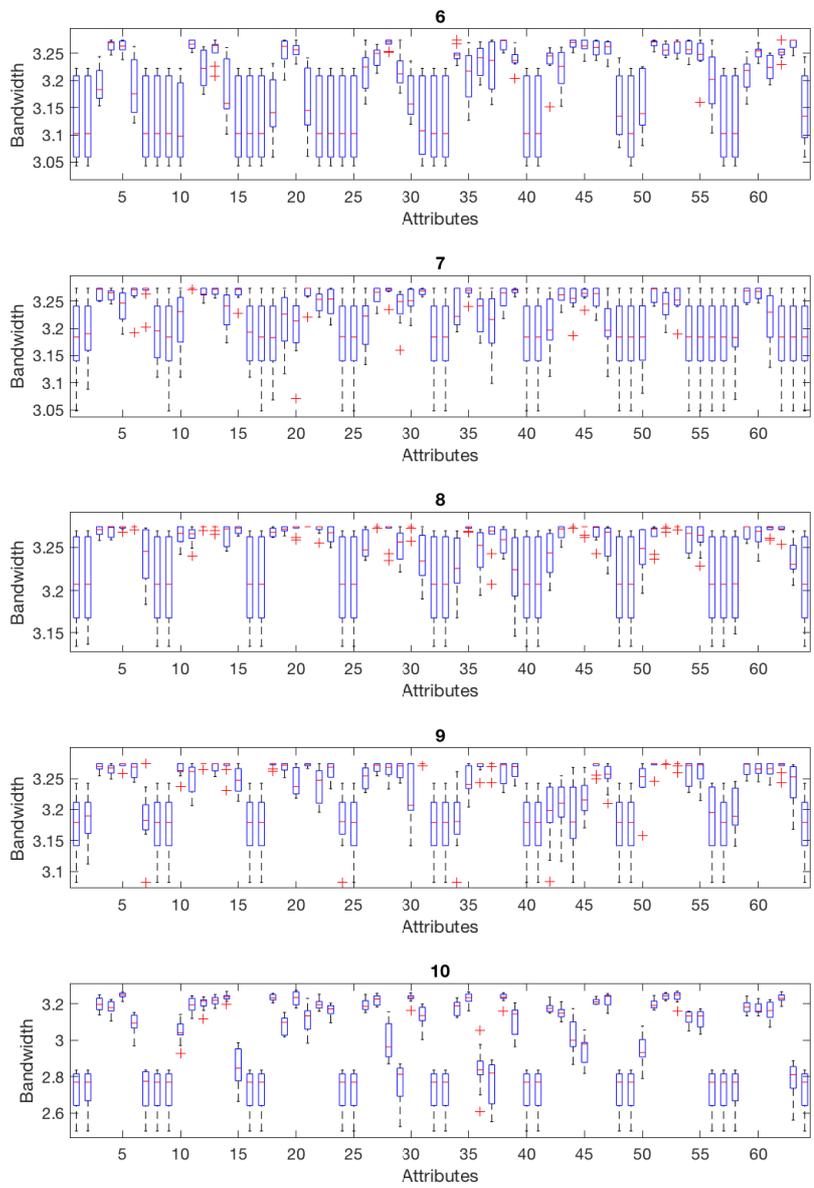}
	\caption{The box-plots of mean predicted bandwidths for handwritten digit 6,7,8,9,0.}
	\label{fig:digit2}

\end{figure}


\section{Asymptotic Properties}

\newtheorem{mydef}{Theorem}
\begin{mydef}
Let $\mathbf{X}\in \mathbb{R}^d$ be compact and $Y \in \mathcal{Y} =\{1, \dots, c\}$ the label variable. $p_{y'} = P(Y=y')$ and $f_{y'}$ denote the unknown population probability and the unknown probability density function for the $y'$th class, respectively, $y' \in \mathcal{Y}$. Based the sparsity assumption that $f_{y'}(\mathbf{x})$  could be factorized into two components, $f_{y'}(\mathbf{x})\propto g_{y'}(x_1, \dots, x_{r_{y'}})b_{y'}(\mathbf{x})$, where the second derivate $b_{y'}^{(jj)}(\mathbf{x}) =0$ for $j=1, \dots,d$.
Let $\hat f_{y'}(\mathbf{x};\hat H_{y'})$ be the estimated probability density function by Rodeo, where $\hat H_{y'}=diag(\hat h_1^{y'},\dots,\hat h_d^{y'})$ is the corresponding estimated bandwidth matrix.
Assume the sample size in class $y'$ and the total sample size are $n_{y'}$ and $n = \sum_{y'\in \mathcal{Y}}n_{y'}$. Use the the sample proportion $\hat p_{y'} = n_{y'}/n$ and Rodeo estimates $\hat f_{y'}(\mathbf{x};\hat H_{y'})$ to construct the classification rule: 
\begin{equation}
\hat y = \argmax_{y' \in \mathcal{Y}}\hat p_{y'}\hat f_{y'}(\mathbf{x};\hat H_{y'}),
\label{AppBayes}
\end{equation}
which is used to approximate the Bayes rule : 
\begin{equation}
y = \argmax_{y' \in \mathcal{Y}}p_{y'}f_{y'}(\mathbf{x}).
\label{Bayes}
\end{equation}
If $L^\star$ and $L_n$ denote the probability of error using the Bayes rule (\ref{Bayes}) and the approximation (\ref{AppBayes}) respectively, then
$L_n- L^*\rightarrow 0$ a.s. 
\label{theorem1}
\end{mydef}

\begin{proof}
Here we show the consistent properties of the results of the proposed multi-class classification problem.
The proposed algorithm uses the estimated conditional probability to approximate the Bayes rule:
\begin{equation}
 y = \argmax _{y' \in \mathcal{Y}} P(Y = y'|\mathbf{X} = \mathbf{x}).
\end{equation} 
Because $P(Y = y'|\mathbf{X} = \mathbf{x}) = \frac{1}{Z}P(Y = y')P(\mathbf{X} = \mathbf{x}|Y=y')$, where $Z = p(\mathbf{X} = \mathbf{x}) = \sum_{y'=1}^cP(\mathbf{X} = \mathbf{x}|Y=y')$ is the scaling factor, the other way to approximate the Bayes classifier is based on the probability models of each group: 
\begin{align*}
y &= \argmax_{y' \in \mathcal{Y}} P(Y = y') P(\mathbf{X} = \mathbf{x}|Y = y')\\
& =\argmax_{y' \in \mathcal{Y}} p_{y'}f_{y'}(\mathbf{x}),
\end{align*} 
where $f_{y'}$ and $p_{y'}$ denote the probability density function and population probability for the $y'$th class respectively. Van Ryzin (1966) has shown the difference of probability of error  using the Bayes rule and approximation are bounded by

\begin{equation}
0 \leqslant L_n- L^* \leqslant \sum_{y' = 1}^c \int \Big|p_{y'}f_{y'}(\mathbf{x}) - \hat p_{y'}\hat f_{y'}(\mathbf{x})\Big| dx\\
\label{eq:errorrate1}
\end{equation}
where $L^\star$ and $L_n$ denote the probability of error using Bayes rule and approximation respectively. In Eq.(\ref{simplified}), the estimated conditional probability is based on the function of class density estimates, so we prove the consistency of the proposed Bayes approximation from Eq.(\ref{eq:errorrate1}) by Van Ryzin (1966).

From Eq.(\ref{eq:errorrate1}), the upper bound of the difference of probabilities of error can be rewritten as 
\begin{align}
L_n - L^\star \leqslant & \sum_{y' = 1}^c \int \Big| p_{y'} f_{y'}(\mathbf{x}) - \hat p_{y'}\hat f_{y'}(\mathbf{x})\Big| dx\nonumber\\
= & \sum_{y' = 1}^c \left[ \int\hat p_{y'}\Big| f_{y'}(\mathbf{x}) - \hat f_{y'}(\mathbf{x})\Big|dx + \int f_{y'}(\mathbf{x})\Big|p_{y'}  - \hat p_{y'} \Big|dx \right].
\label{eq:errorrate2}
\end{align}
It is natural to use the sample proportions as the estimates of unknown population probabilities $\hat p_{y'} = n_{y'}/n$, $y' = 1, \dots, c$, where $n_{y'}$ is the number of training samples in class $y'$. In this case, the convergence rate of $|p_{y'}  - \hat p_{y'}|$ is $\sqrt{\log(\log(n))/n}$. 
Because $f_{y'}(\mathbf{x})$ is a density function, we have
\begin{equation}
\int f_{y'}(\mathbf{x})\Big|p_{y'}  - \hat p_{y'} \Big|dx = \mathcal{O}(\sqrt{\log(\log(n))/n}).
\label{eq:ConvergenceRate1}
\end{equation}

Consider the kernel estimates of the density $f_{y'}$ are
\begin{equation}
\hat f_{y'}(\mathbf{x}|H_{y'}) = \frac{1}{n_{y'}}\sum_{i=1}^{n_{y'}}\frac{1}{\det(H_{y'})}K(inv(H_{y'})(\mathbf{x}-\mathbf{x}_i)),
\end{equation}
where 
$K$ is a $d$-dimensional bounded symmetric kernel satisfying
\begin{equation}
\int K(u)du = 1_d, \text{ and } \int uK(u)du = 0_d, 
\end{equation}
and $H_{y'} = diag(h^{y'}_1, \dots, h^{y'}_d)$ is the bandwidth matrix. 
Let 
\begin{equation}
\alpha = \int_{\mathbb{R}^d} uu^TK(u)du \text{ and } \beta = \Big(\int_{\mathbb{R}^d} K^2(u)du \Big)^{1/2}.
\label{eq:alphabeta}
\end{equation}
If $f_{y'}$ is bounded and if all the second derivates of $f_{y'}$ are bounded and continuous, then by Devroye and Gyorfi (1985) \cite{devroye1985} and Hall and Wand (1988) \cite{HALL198859}
\begin{align}
&E\int |\hat f_{y'}(\mathbf{x}) - f_{y'}(\mathbf{x}) | dx\nonumber \\
\leqslant & \frac{\alpha}{2}\int tr(H_{y'}^{T}\mathcal{H}_{f_{y'}(\mathbf{x})}(\mathbf{x})H_{y'}) + \beta(n_{y'}\det (H_{y'}))^{-1/2} \int \sqrt{f_{y'}}\nonumber \\
&+ o(tr(H_{y'}^{T} H_{y'}) + (n_{y'}\det (H_{y'}))^{-1/2}  ),
\end{align}
where $\mathcal{H}_{f_{y'}(\mathbf{x})}(\mathbf{x})$ denotes the Hessian matrix of $f_{y'}(\mathbf{x})$.

Given the sparsity assumption that $f_{y'}(\mathbf{x})$  could be factorized into two components, $f_{y'}(\mathbf{x})\propto g_{y'}(x_1, \dots, x_{r_{y'}})b_{y'}(\mathbf{x})$, where the second derivate $b_{y'}^{(jj)}(\mathbf{x}) =0$ for $j=1, \dots,d$, the Rodeo algorithm outputs the selected bandwidths $\hat H_{y'} = diag(\hat h^{y'}_1, \dots, \hat h^{y'}_d)$ that satisfies
\begin{equation*}
\text{when }\lim _{n\rightarrow\infty} \hat h^{y'}_j =0; \lim _{n\rightarrow\infty} n\hat h^{y'}_j =\infty; j = 1,\dots, d, 
\end{equation*}
\begin{equation}
P\Big(\hat h^{y'}_j = h^{(0)} \text{ for all }j > r_{y'} \Big)\rightarrow 1,
\end{equation}
and
\begin{equation}
P\Big(h^{(0)}(nb_n)^{-1/(4+r_{y'})} \leqslant \hat h^{y'}_j \leqslant h^{(0)}(na_n)^{-1/(4+r_{y'})} \text{ for all }j \leqslant r_{y'} \Big)\rightarrow 1,
\end{equation}
where $\lim\inf_n |\frac{a_n}{logn}|>0$, $b_n = \mathcal{O}(\log n)$, and $h^{(0)} = c_0/(\log \log n)$ for some constant $c_0$.
Because the convergence rate for $\hat h_j^{y'}$ is $n_{y'}^{-1/(4+r_{y'})}$, $j \leqslant r_{y'}$, we can set $\hat h_j^{y'} = k_j^{y'}n_{y'}^{-1/(4+r_{y'})}$ for some constant k. In this case, 
\begin{align}
&E\int |\hat f_{y'}(\mathbf{x}|\hat H_{y'}) - f_{y'}(\mathbf{x}) | dx\nonumber \\
\leqslant & \frac{\alpha}{2}\int tr(\hat H_{y'}^{T}\mathcal{H}_{\mathcal{R}_{y'}}(\mathbf{x})\hat H_{y'}) + \beta(n_{y'}\det (\hat H_{y'}))^{-1/2} \int \sqrt{f_{y'}}\nonumber \\
&+ o(tr(\hat H_{y'}^{T} \hat H_{y'}) + (n_{y'}\det (\hat H_{y'}))^{-1/2}  )\nonumber\\
=& \Bigg(\frac{\alpha}{2} \int \Big|\sum_{j=1}^{r_{y'}}k_j^{y'2}f_{y'}^{(jj)}(\mathbf{x})\Big|  + \beta \Big(\prod_{j=1}^{r_{y'}}k_j^{y'}\Big)^{-1/2}\int \sqrt{f_{y'}}\Bigg)n_{y'}^{-\frac{2}{4+r_{y'}}} + o\Big(n_{y'}^{-\frac{2}{4+r_{y'}}}\Big),
\label{eq:ConvergenceRate2}
\end{align}
where $\mathcal{H}_{\mathcal{R}_{y'}}(\mathbf{x})$ is the Hessian matrix of the relevant dimension $j\leqslant r_{y'}$.
If Eq.(\ref{eq:ConvergenceRate2}) holds almost surely, by Kundu and Martinsek (1997) \cite{Kundu1997}, it means
\begin{align}
&\limsup_{n\rightarrow\infty}\Bigg(\int |\hat f_{y'}(\mathbf{x}|\hat H_{y'}) - f_{y'}(\mathbf{x})\Bigg)n_{y'}^{\frac{2}{4+r_{y'}}}\nonumber\\
\leqslant &\frac{\alpha}{2} \int \Big|\sum_{j}^{r_{y'}}k_j^{y'2}f_{y'}^{(jj)}(\mathbf{x})\Big| +  \beta \Big(\prod_{j=1}^{r_{y'}}k_j^{y'}\Big)^{1/2}\int \sqrt{f_{y'}}\quad\text{   a.s.}
\label{eq:ConvergenceRate3}
\end{align}

Let $r$ be the max value of the numbers of relevant variables among $c$ groups: $r = \max_{y'=1}^c r_{y'}$. From Eq.(\ref{eq:errorrate2}), Eq.(\ref{eq:ConvergenceRate1}), and Eq.(\ref{eq:ConvergenceRate3})
\begin{align}
L_n - L^\star &\leqslant \sum_{y' = 1}^c \Bigg( \hat p_{y'}\int \Big| f_{y'}(\mathbf{x}) - \hat f_{y'}(\mathbf{x})\Big|dx + \int f_{y'}(\mathbf{x})\Big|p_{y'}  - \hat p_{y'} \Big|dx \Bigg)\nonumber\\
&=\sum_{y' = 1}^c \Bigg(\mathcal{O}\Big(n_{y'} ^{-2/(4+r_{y'})}\Big)+ \mathcal{O}\Big(\sqrt{\log(\log(n))/n}\Big)\Bigg)\nonumber\\
&= \mathcal{O}\Big(n^{-2/(4+r)}\Big),
\label{eq:errorrate3}
\end{align}  
because $\sqrt{\log(\log(n))/n}$ is faster than $n^{-2/(4+r)}$.
Therefore, the consistency of the classification procedures using the proposed algorithm is proven:
\begin{equation}
L_n - L^\star \xrightarrow[n\rightarrow \infty] {}0 \text{ a.s.}
\end{equation}
\end{proof}

\newtheorem{mycol}{Collary}
\begin{mycol}
From the consistency property in Theorem \ref{theorem1}, in order to find the desired samples sizeses for approaching the convergence level, progressively increase sample sizes such that the function of size $n_{y'}^{(2+r_{y'})/(4+r_{y'})}$ proportion to $\frac{\alpha}{2} \int \Big|\sum_{j=1}^{r_{y'}}k_j^{y'2}f_{y'}^{(jj)}\Big|  + \beta \Big(\prod_{j=1}^{r_{y'}}k_j^{y'}\Big)^{-1/2}\int \sqrt{f_{y'}}$, $j = 1, \dots, d$, $y' = 1, \dots, c$, can help speed up the procedure, where $\alpha$ and $\beta$ are defined on Eq.(\ref{eq:alphabeta}), $f_{y'}$ is the density function of the $y'$th group, $k_j^{y'}$ is some constant that is proportion to the ratio of the kernel bandwidth $h_j^{y'}$ over $n_{y'}^{-1/(4+r_{y'})}$, and $r^{y'}$ is the number of related variables in $y'$th group.
\end{mycol}
From Eq.(\ref{eq:errorrate3}), if we want to bound the difference of error rates by $\epsilon: L_n - L^\star \leqslant \epsilon$, asymptotically, it would be enough to bound
\begin{equation}
\sum_{y' = 1}^c \hat p_{y'}\int \Big| f_{y'}(\mathbf{x}) - \hat f_{y'}(\mathbf{x})\Big|dx.
\end{equation}
Therefore, according to Eq.(\ref{eq:ConvergenceRate3}), the appropriate sample sizes for the training samples in each groups should follow 
\begin{equation}
\sum_{y' = 1}^c \frac{n_{y'}}{n}\Bigg(\frac{\alpha}{2} \int \Big|\sum_{j=1}^{r_{y'}}k_j^{y'2}f_{y'}^{(jj)}\Big|  + \beta \Big(\prod_{j=1}^{r_{y'}}k_j^{y'}\Big)^{-1/2}\int \sqrt{f_{y'}}\Bigg)n_{y'}^{-\frac{2}{4+r_{y'}}}\leqslant \epsilon,
\label{eq:samplesize1}
\end{equation}
which can be rewritten as
\begin{equation}
\sum_{y' = 1}^c n_{y'}^{\frac{2+r_{y'}}{4+r_{y'}}}\Bigg(\frac{\alpha}{2} \int \Big|\sum_{j=1}^{r_{y'}}k_j^{y'2}f_{y'}^{(jj)}\Big|  + \beta \Big(\prod_{j=1}^{r_{y'}}k_j^{y'}\Big)^{-1/2}\int \sqrt{f_{y'}}\Bigg)\leqslant n\epsilon.
\label{eq:samplesize2}
\end{equation}
Let $\vec{A} = (A_1, \dots, A_c)'$ and $\vec{B} = (B_1, \dots, B_c)'$  denote vectors with components
\begin{align}
A_{y'} = & n_{y'}^{(2+r_{y'})/(4+r_{y'})}
\label{eq:A1}\\
B_{y'} = & \frac{\alpha}{2} \int \Big|\sum_{j=1}^{r_{y'}}k_j^{y'2}f_{y'}^{(jj)}\Big|  + \beta \Big(\prod_{j=1}^{r_{y'}}k_j^{y'}\Big)^{-1/2}\int \sqrt{f_{y'}}, \;\; y' = 1 \dots, c
\label{eq:B1}
\end{align}
respectively. Then Eq.(\ref{eq:samplesize2}) is 
\begin{equation}
\sum_{y' = 1}^c A_{y'} B_{y'}  = \vec{A}\cdot\vec{B}\leqslant n\epsilon.
\end{equation}
Because the max value of the inner product of two vectors happens when they are parallel to each other, bound the difference of error rates with the sample sizes according the parallel setting $A \propto B$ can help decide the sampling procedures to meet the asymptotic rate.

However, the true density functions $f_{y'}, y' = 1,\dots, c$ are unknown. In this situation, the estimated density function by Rodeo $\hat f_{y'}$ are used and the integral is replaced by the Monte Carlo integral using importance sampling on samples which are predicted as the given group. The components in vectors $\vec A$ and $\vec B$ are replaced as
\begin{align}
\hat A_{y'} = & n_{y'}^{(2+\hat r_{y'})/(4+\hat r_{y'})}
\label{eq:A2}\\
\hat B_{y'} = & \frac{\alpha}{2} \sum_{\{\mathbf{x}^{\hat y'}\}} \frac{\Big|\sum_{j=1}^{\hat r_{y'}}k_j^{y'2}\hat f_{y'}^{(jj)}(\mathbf{x})\Big|}{\hat f_{y'}(\mathbf{x})}  + \beta \Big(\prod_{j=1}^{\hat r_{y'}}k_j^{y'}\Big)^{-1/2}\sum_{\{\mathbf{x}^{\hat y'}\}} \frac{\sqrt{\hat f_{y'}(\mathbf{x})}}{\hat f_{y'}(\mathbf{x})}.
\label{eq:B2}
\end{align}

In this case, a 2-step algorithm for finding the sample sizes to approach the desired convergence level is proposed. It includes the estimation and resampling steps. In the estimation step (E-step), given the current training data, with applying the Rodeo density estimation for all groups, we have estimated bandwidths for each group and then get the estimated label for testing samples. In the resampling step (R-step), based on the density estimation and estimated label, we can make decision to include more training samples to meet the condition $A \propto B$ for fastening the procedure. Then using Eq.(\ref{eq:samplesize2}) as the stopping rule, when the sample sizes $n_1, \dots, n_c$ meets the criteria, we think take the final density estimation for classification is good enough to approach the error risk of Bayes rule. The algorithm is given below.\\

\begin{algorithm}[H]
\SetAlgoLined
  \caption{Sample Sizes Estimation}
  \label{alg4}
    \KwData{
    $\{\mathbf{x}^{y'}\}$, 
      ${y'}= 1, \dots, c$:  data set in $y'$th  group\\
    }
    \KwIn{$n_0$: initial training sample size for each group,\\
    \hspace{41pt}$n_{test}$ : number of testing samples in each group,\\
    \hspace{41pt}$\epsilon^\star$: upper bound}
    \KwOut {$n_1, \dots, n_c$: sample sizes needed for each group }
    \vspace{0.25in}
    \textbf{Initialization}\\
    $n_0$ training samples from each group: $\{\mathbf{x}^{y'}_1, \dots \mathbf{x}^{y'}_{n_0}\}$, $y' = 1, \dots, c,$\\
    $\{\mathbf{x}_i, i = 1, \dots, m=c*n_{test}\}$: choose $n_{test}$ samples form each group to form the testing set,\\
    $n_y = n_0, N = \sum n_{y'}, \epsilon = 1$\\

    \vspace{0.25in}
    \While{$\epsilon > \epsilon^\star$ }{
	\textbf{E-step: Density Estimation}\\
	Apply classification with Rodeo algorithm (Algorithm \ref{C+R}) on testing samples, we have\\
	\begin{enumerate}
	\item Estimated label: $\hat y_i$ for each testing sample
	\item Selected variables: $R^{y'}$ for each group
	\item Size of related variables: $\hat r^{y'} = size(R^{y'})$ for each group
	\end{enumerate}
%
%
%
%
	\textbf{R-step: Resampling}\\

	Calculate $\hat A_{y'}$ and $\hat B_{y'}$ by Eq.(\ref{eq:A2}) and Eq.(\ref{eq:B2}), $y' = 1, \dots, c$\\

	Calculate $\epsilon = \hat{\vec A} * \hat{\vec B}/N$\\
	Increase total sample size: $N = N + N_{add}$\\
	Rearrange $(n_1, \dots, n_c) \ni\hat{\vec A}\propto \hat{\vec B}$ and $\sum_{y'=1}^c n_{y'} = N$
	
	}
\end{algorithm}

\bibliographystyle{spmpsci}      
\bibliography{report}   

\end{document}